\documentclass[12pt]{article}

\usepackage{mathpazo}
\usepackage{overpic}
\usepackage{rotating}

\usepackage[colorlinks, linkcolor=blue, citecolor=blue]{hyperref}
\usepackage{color}
\usepackage{graphicx,subfigure,amsmath,amssymb,amsfonts,bm,epsfig,url,dsfont}
\usepackage{amsthm}
\usepackage{tikz}
\usepackage{bbm}      
\usepackage{booktabs}
\usepackage{cases}
\usepackage{fullpage}
\usepackage[small,bf]{caption}

\usepackage[top=1in,bottom=1in,left=1in,right=1in]{geometry}
\usepackage{fancybox}
\usepackage{algorithm}
\usepackage{verbatim}
\usepackage{algorithmic}

\usepackage{mathtools}

\usepackage{scalerel,stackengine}
\stackMath
\newcommand\reallywidehat[1]{%
\savestack{\tmpbox}{\stretchto{%
  \scaleto{%
    \scalerel*[\widthof{\ensuremath{#1}}]{\kern-.6pt\bigwedge\kern-.6pt}%
    {\rule[-\textheight/2]{1ex}{\textheight}}
  }{\textheight}%
}{0.5ex}}%
\stackon[1pt]{#1}{\tmpbox}%
}

\newcommand{\ceil}[1]{\left \lceil #1 \right \rceil}

\newtheorem{definition}{Definition}

\newtheorem{theorem}{Theorem}

\newtheorem{lemma}{Lemma}
\newtheorem{prop}{Proposition}

\newtheorem{rmk}{Remark}

\newenvironment{fminipage}%
  {\begin{Sbox}\begin{minipage}}%
  {\end{minipage}\end{Sbox}\fbox{\TheSbox}}

\makeatletter
\newcommand*{\rom}[1]{\expandafter\@slowromancap\romannumeral #1@}
\makeatother

\newcommand{\Ind}{\mathbbm{1}}

\newcommand{\vol}{\s{Vol}}

\newcommand{\abs}[1]{\left|#1\right|}
\newcommand{\R}{\mathbb{R}} 
\newcommand{\N}{\mathbb{N}}

\newcommand{\E}{\mathbb{E}}

\def\P{{\mathbb P}}
\def\S{{\mathbb S}}

\newcommand {\pr} {\mathbb{P}}

\newcommand{\calA}{{\cal A}}
\newcommand{\calB}{{\cal B}}

\newcommand{\calH}{{\cal H}}

\newcommand{\calL}{{\cal L}}

\newcommand{\calN}{{\cal N}}

\newcommand{\calP}{{\cal P}}
\newcommand{\calQ}{{\cal Q}}

\newcommand{\be}{\begin{equation}}
\newcommand{\ee}{\end{equation}}
\newcommand{\beqna}{\begin{eqnarray}}
\newcommand{\eeqna}{\end{eqnarray}}

\DeclarePairedDelimiterX{\set}[1]{\{}{\}}{\setargs{#1}}
\DeclarePairedDelimiter\abss{\lvert}{\rvert}

\DeclarePairedDelimiter\parenv{\lparen}{\rparen}

\newcommand{\indep}{\perp \!\!\! \perp}
\DeclareMathOperator{\Id}{\s{Id}}

\newcommand{\p}[1]{\left(#1\right)}
\newcommand{\pp}[1]{\left[#1\right]}
\newcommand{\ppp}[1]{\left\{#1\right\}}
\newcommand{\norm}[1]{\left\|#1\right\|}

\usepackage{xspace}

\newcommand{\s}[1]{\mathsf{#1}}

\allowdisplaybreaks
\begin{document}
\title{Phase Transitions in the Detection of Correlated Databases}

\author{Dor~Elimelech\thanks{D. Elimelech is with the School of Electrical and Computer Engineering at Ben-Gurion university, {B}eer {S}heva 84105, Israel (e-mail:  \texttt{doreli@post.bgu.ac.il}). The work of D.Elimelech was supported by the ISRAEL SCIENCE FOUNDATION (grant No. 1058/18).}~~~~~~~~~~~Wasim~Huleihel\thanks{W. Huleihel is with the Department of Electrical Engineering-Systems at Tel Aviv university, {T}el {A}viv 6997801, Israel (e-mail:  \texttt{wasimh@tauex.tau.ac.il}). The work of W. Huleihel was supported by the ISRAEL SCIENCE FOUNDATION (grant No. 1734/21).}}

\maketitle

\begin{abstract}

We study the problem of detecting the correlation between two Gaussian databases $\s{X}\in\mathbb{R}^{n\times d}$ and $\s{Y}^{n\times d}$, each composed of $n$ users with $d$ features. This problem is relevant in the analysis of social media, computational biology, etc. We formulate this as a hypothesis testing problem: under the null hypothesis, these two databases are statistically independent. Under the alternative, however, there exists an unknown permutation $\sigma$ over the set of $n$ users (or, row permutation), such that $\s{X}$ is $\rho$-correlated with $\s{Y}^\sigma$, a permuted version of $\s{Y}$. We determine sharp thresholds at which optimal testing exhibits a phase transition, depending on the asymptotic regime of $n$ and $d$. Specifically, we prove that if $\rho^2d\to0$, as $d\to\infty$, then weak detection (performing slightly better than random guessing) is statistically impossible, \emph{irrespectively} of the value of $n$. This compliments the performance of a simple test that thresholds the sum all entries of $\s{X}^T\s{Y}$. Furthermore, when $d$ is fixed, we prove that strong detection (vanishing error probability) is impossible for any $\rho<\rho^\star$, where $\rho^\star$ is an explicit function of $d$, while weak detection is again impossible as long as $\rho^2d\to0$. These results close significant gaps in current recent related studies.
    
\end{abstract}

\section{Introduction} \label{sec:intro}
\sloppy
Database alignment and the quantification of the relation between different databases are among the most fundamental tasks in modern applications of statistics. In many cases, the observed databases are high-dimensional, unlabeled, noisy, and scrambled, making the task of inference challenging. An example of such an inference task between, say, two databases, formulated as an hypothesis testing problem, is the following: under the null hypothesis, the databases are statistically uncorrelated, while under the alternative, there exists a permutation which scrambles one database, and for which the two databases are correlated. Then, observing the databases, under what conditions can we tell if they are correlated or not?

It turns out that the above general inference formulation is relevant in many fields, such as, computational biology \cite{pmid:18725631,kang2012fast}, social network analysis \cite{4531148,5207644}, computer vision \cite{Berg2005,10.5555/2976456.2976496}, and data anonymization/privacy based systems. A concrete classical example is: consider two data sources (e.g., Netflix and IMDb), each supplying lists of features for a set of entities, say, users. Those features might be various characteristics of those users, such as, names, user identifications, ratings. In many cases, reliable labeling for features is either not available or deleted (so as to hide sensitive unique identifying information) due to privacy concerns. In general, this precludes trivial identification of feature pairs from the two sources that correspond to the same user. Nonetheless, the hope is that if the correlation between the two databases is sufficiently large, then it is possible to identify correspondences between the two databases, and generate an alignment between the feature lists \cite{4531148,5207644}.

Quite recently, the \emph{data alignment problem}, a seemingly simple probabilistic model which captures the scenario above, was introduced and investigated in \cite{10.1109/ISIT.2018.8437908,pmlr-v89-dai19b}. Specifically, there are two databases $\s{X}\in\mathbb{R}^{n\times d}$ and $\s{Y}^{n\times d}$, each composed of $n$ users with $d$ features. There exist an unknown permutation/correspondence which match users in $\s{X}$ to users in $\s{Y}$. For a pair of matched database entries, the features are dependent according to a known distribution, and, for unmatched entries, the features are independent. The goal is to \emph{recover} the unknown permutation, and derive statistical guarantees for which recovery is possible and impossible, as a function of the correlation level, $n$, and $d$. Roughly speaking, this recovery problem is well-understood for a wide family of probability distributions. For example, in the Gaussian case, denoting the correlation coefficient by $\rho$, it was shown in \cite{pmlr-v89-dai19b} that if $\rho^2 = 1-o(n^{-4/d})$ then perfect recovery is possible, while impossible if $\rho^2 = 1-\omega(n^{-4/d})$, as $n,d\to\infty$.

The \emph{detection} counterpart of the above recovery problem was also investigated in \cite{9834731,nazer2022detecting}. Here, as mentioned above, the underlying question is, given two databases, can we determine whether they are correlated? It was shown that if $\rho^2d\to\infty$ then efficient detection is possible (with vanishing error probability), simply by thresholding the sum of entries of $\s{X}^T\s{Y}$. On the other hand, it was also shown that if $\rho^2d\sqrt{n}\to0$ and $d=\Omega(\log n)$, then detection is information-theoretically impossible (i.e., lower bound), again, as $n,d\to\infty$. Unfortunately, it is evident that there is a substantial gap between those two bounds. Most notably, the aforementioned upper bound is completely independent of $n$, implying that it does not play any significant role in the detection problem, while the lower bound depends on $n$ strongly. This sets precisely the main goal of this paper: \emph{we would like to characterize the detection boundaries tightly, as a function of $n$ and $d$.}

At first glance, one may suspect that the source for this gap is the proposed algorithm. Indeed, note that under the alternative distribution, the latent permutation represents a \emph{hidden} correspondence under which the databases are correlated. Accordingly, the ``thresholding the sum" approach ignores this hidden combinatorial structure, and therefore, seemingly suboptimal.  However, it turns out that in the regime where $d\to\infty$, \emph{independently} of $n$, this simple approach is actually surprisingly optimal: we prove that whenever $\rho^2d\to0$, then weak detection (performing slightly better than random guessing) is information-theoretically impossible, while if $\rho^2\lessapprox1/d$ (that is, $\rho^2\leq (1-\varepsilon)/d$ for some $\varepsilon>0$), then strong detection (vanishing error probability) is information-theoretically impossible. This behaviour, however, changes when $d$ is fixed. 
In this case, we prove that strong detection is impossible for any $\rho<\rho^\star$, where $\rho^\star$ is an explicit function of $d$, while possible when $\rho=1-o(n^{-2/(d-1)})$. The later is achieved by counting the number of empirical pairwise correlations that exceed as certain threshold. Finally, we prove that weak detection (performing slightly better than random guessing) is impossible when $\rho^2d\to0$, while possible for any $\rho>\rho^{\star\star}$, where  $\rho^{\star\star}$ is again an explicit function of $d$. The bounds from previous work and our new results described above are summarized in Table~\ref{tab:exact}.

\begin{table}
\begin{center}
\renewcommand{\arraystretch}{2}
\begin{tabular}{ |p{3.4cm}||p{2.3cm}|p{2.2cm}|p{2.35cm}| p{2.2cm}|}
 \hline
   & \multicolumn{2}{|c|}{\textbf{Weak Detection}} &  \multicolumn{2}{|c|}{\textbf{Strong Detection}}  \\
 \hline
  \textbf{Asymptotics} & \textbf{Possible} & \textbf{Impossible} &\textbf{Possible}& \textbf{Impossible}\\
 \hline
 $n,d\to \infty$   & $\Omega\parenv*{d^{-1}}^* $    & $o\parenv*{d^{-1}} $ &  $\omega\parenv*{d^{-1}}^*$    & $(1-\varepsilon)d^{-1} $\\
 $d\to \infty$, $n$ constant & $\Omega\parenv*{d^{-1}}^*$    & $o\parenv*{d^{-1}} $ &  $\omega\parenv*{d^{-1}}^*$    & $O\parenv*{d^{-1}} $\\
 $n\to\infty$, $d$ constant &$\rho^2\geq \frac{60\log2}{d}^*$ & $o\parenv*{1}$&  
 $1-o(n^{-\frac{2}{d-1}})$ &$\rho^\star(d)$\\
\hline
\end{tabular}
\end{center}
\caption{A summary of our bounds on $\rho^2$, for weak and strong detection, as a function of the asymptotic regime. Bounds marked with $*$ follows from the upper-bound of \cite{nazer2022detecting}.}
\label{tab:exact}
\end{table}

We now mention other related work briefly. In \cite{9174507} the problem of partial recovery of the permutation aligning the databases was analyzed. In \cite{ShiraniISIT} necessary and sufficient conditions for successful recovery matching using a typicality-based framework were established. Furthermore, \cite{Bakirtas2021DatabaseMU} and \cite{Bakirtas2022DatabaseMU} proposed and explored the problem of permutation recovery under feature deletions and repetitions, respectively. Recently, the problem of joint correlation detection and alignment of {G}aussian database was analyzed in \cite{tamir}. Finally, we note that the problem of database alignment and detection is closely related to a wide verity of planted matching problems, specifically, the graph alignment problem, with many exciting and interesting results, and useful mathematical techniques, see, e.g., \cite{GraphAl1,GraphAl2,GraphAl3,GraphAl4,wu2020testing,GraphAl5,GraphAl6,GraphAl7}, and many references therein.

\paragraph{Notation.}
For any $n\in\mathbb{N}$, the set of integers $\{1,2,\dots,n\}$ is denoted by $[n]$. Let $\mathbb{S}_n$ denotes the set of all permutations on $[n]$. For a given permutation $\sigma\in\mathbb{S}_n$, let $\sigma_i$ denote the value to which $\sigma$ maps $i\in[n]$. We use $\log$ to denote the natural logarithm function (i.e., of base $e$). Random vectors are denoted by capital letters such as $X$ with transpose $X^T$. A collection of $n$ random vectors is written as $\s{X}$ = $(X_1,\ldots,X_n)$. The notation $(X_1,\ldots,X_n)\sim P_X^{\otimes n}$ means that the random vectors $(X_1,\ldots,X_n)$ are independent and identically distributed (i.i.d.) according to $P_X$. We use $\calN(\eta,\Sigma)$ to represent the multivariate normal distribution with mean vector $\eta$ and covariance matrix $\Sigma$. Let $\s{Poisson}(\lambda)$ denote the Poisson distribution with parameter $\lambda$. The $n\times n$ identity matrix is denoted by $I_{n\times n}$, and $0_d$ denotes the all zero $d$-dimensional column vector. Let $\calL(Y)$ denote the law, that is, the probability distribution, of a random variable $Y$. For probability measures $\mathbb{P}$ and $\mathbb{Q}$, let $d_{\s{TV}}(\mathbb{P},\mathbb{Q})=\frac{1}{2}\int |\mathrm{d}\mathbb{P}-\mathrm{d}\mathbb{Q}|$ denote the total variation distance. For a probability measure $\mu$ on a space $\Omega$, we use $\mu^{\otimes d}$ for the product measure of $\mu$ ($d$ times) on the product space $\Omega^d$. For a measure $\nu\ll\mu$ (that is, a measure absolutely continuous with respect to $\mu$), we denote (by abuse of notation) the Randon-Nikodym derivative $\nu$ with respect to $\mu$ by $\frac{\nu}{\mu}$. For functions $f,g:\N \to \R$, we say that $f=O(g)$ (and  $f=\Omega(g)$) if there exists $c>0$ such that $f(n)\leq cg(n)$ (and $f(n)\geq cg(n)$) for all $n$. We say that $f=o(g)$ if $\lim_{n\to\infty}f(n)/g(n)=0$, and that $f=\omega(g)$ if $g=o(f)$. 

\section{Problem Formulation and Main Results}\label{sec:model}
\paragraph{Probabilistic Model.} Consider the following binary hypothesis testing problem. Under the null hypothesis $\calH_0$, the Gaussian databases $\s{X}$ and $\s{Y}$ are generated independently with $X_1,\ldots,X_n,Y_1,\ldots,Y_n\sim \calN(0_d,I_{d\times d})$. Let $\mathbb{P}_0$ denote the resulting distribution over $(\s{X},\s{Y})$. Under the alternate hypothesis $\calH_1$, the databases $\s{X}$ and $\s{Y}$ are correlated with permutation $\sigma$ for some unknown $\sigma\in\mathbb{S}_n$ and some known correlation coefficient $\rho\in[-1,1]\setminus\{0\}$. Namely,
\begin{equation}
\begin{aligned}\label{eqn:decprob}
    &\calH_0: (X_1,Y_1),\ldots,(X_n,Y_n)\stackrel{\mathrm{i.i.d}}{\sim} \mathcal{N}^{\otimes d}\p{\begin{bmatrix}
0 \\
0 
\end{bmatrix},\begin{bmatrix}
1 & 0\\
0 & 1
\end{bmatrix}}\\
& \calH_1: (X_1,Y_{\sigma_1}),\ldots,(X_n,Y_{\sigma_n})\stackrel{\mathrm{i.i.d}}{\sim} \mathcal{N}^{\otimes d}\p{\begin{bmatrix}
0 \\
0 
\end{bmatrix},\begin{bmatrix}
1 & \rho\\
\rho & 1
\end{bmatrix}},
\end{aligned}
\end{equation}
for some permutation $\sigma\in\mathbb{S}_n$. For a fixed $\sigma\in \S_n $, we denote the joint distribution measure of $(\s{X},\s{Y})$  under the hypothesis $
\calH_{1}$ by $\P_{\calH_1\vert\sigma}$. See Figure~\ref{fig:comp} for a visual illustration of our probabilistic model.

\begin{figure}[t]
\centering
\begin{overpic}[scale=0.4]
    {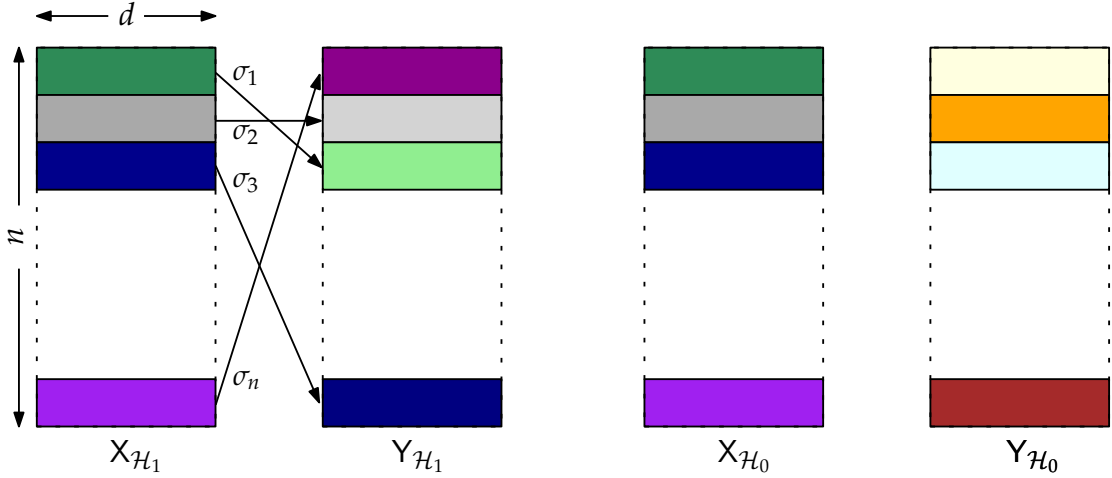}
    \put(-0.5,21.7){\begin{turn}{90}$n$\end{turn}}
    \put(9.6,41.8){$d$}
    \put(8.8,2){$\s{X}_{\calH_1}$}
    \put(34.5,2){$\s{Y}_{\calH_1}$}
    \put(64.2,2){$\s{X}_{\calH_0}$}
    \put(90.5,2){$\s{Y}_{\calH_0}$}
    \put(90.5,2){$\s{Y}_{\calH_0}$} 
    \put(20,36.8){\small{$\sigma_1$}}
    \put(20,31.4){\small{$\sigma_2$}} 
    \put(20,27){\small{$\sigma_3$}}     
    \put(20,9){\small{$\sigma_n$}}    
\end{overpic}
\caption{
An illustration of the probabilistic model. On the left are the databases $\s{X}$ and $\s{Y}$ under the hypothesis $\calH_1$, where correlated vectors are marked with a similar color. On the right are the uncorrelated databases under the null hypothesis.  
}
\label{fig:comp}
\end{figure}

\paragraph{Learning Problem.} A test function for our problem is a function $\phi:\R^{n\times d}\times\R^{n\times d}\to \{0,1\}$, designed to determine which of the hypothesis $\calH_0,\calH_1$ occurred. The risk of a test $\phi$ is defined as the sum of its Type-I and (worst-case) Type-II error probabilities, i.e.,
\begin{align}
\s{R}(\phi)\triangleq \P_{\calH_0}[\phi(\s{X},\s{Y})=1]+\max_{\sigma\in \S_n}\P_{\calH_1\vert\sigma}[\phi(\s{X},\s{Y})=0].
\end{align}
The \textit{minimax} risk for our hypothesis detection problem is 
\begin{align}
\s{R}^\star\triangleq\inf_{\phi:\R^{n\times d}\times\R^{n\times d}\to \{0,1\}}\s{R}(\phi).
\end{align}
We remark that $\s{R}$ is a function of $\rho,d$, and $n$, however, we omit them from our notation for the benefit of readability. 

We study the possibility and impossibility of our detection problem in multiple asymptotic regimes. These regimes are characterized by sequences of the parameters $(\rho,d,n)=(\rho_k,d_k,n_k)_{k\in \N}$. We consider the scenarios where $d_k$ and $n_k$ are either bounded or diverge to infinity. Accordingly, whenever we use asymptotic notations, such as, $\rho^2 = o(\cdot)$, $\rho^2=\Omega(\cdot)$, etc., it should be understood in the context of the sequences above. For example, the condition $\rho^2=o(d^{-1})$ means that the sequence $(\rho,d,n)$ satisfies $ \rho^2_k d_k\to0$, as $k\to\infty$.

\begin{definition}
A sequence $(\rho,d,n)=(\rho_k,d_k,n_k)_k$ is said to be:
\begin{enumerate}
    \item Admissible for strong detection if 
    $\lim_{k\to \infty}\s{R}^\star=0$.
    \item Admissible for weak detection if
    $\limsup_{k\to \infty} \s{R}^\star<1$.
\end{enumerate}
Clearly, admissibility of strong detection implies the admissibility of weak detection.
\end{definition}
While admissibility of strong detection clearly refers to the existence of algorithms that correctly detects with probability that tends to $1$, weak detection implies the the existence of algorithms which are asymptotically better then randomly guessing which of the hypothesis occurred. A useful way to rule out the possibility of weak/strong detection is by considering the relaxed average-case problem, where the permutation is uniformly drawn rather then been arbitrary. In that case, the risk function is characterized by the total variation distance between the null hypothesis distribution $\P_{\calH_0}$ and the distribution under the alternative hypothesis $\P_{\calH_1}$. In particular, it can be shown that,
\begin{equation}\label{eq:TV1}
    d_{\s{TV}}(\P_{\calH_0},\P_{\calH_1})=o(1)\implies \lim_{k\to\infty}\s{R}^\star=1,
\end{equation}
and
\begin{equation}\label{eq:TV2}
    d_{\s{TV}}(\P_{\calH_0},\P_{\calH_1})\leq1-\Omega(1) \implies\liminf_{k\to\infty} \s{R}^\star>0,
\end{equation}
which correspond to the impossibility of weak and strong detection, respectively. We further discuss the relations specified in \eqref{eq:TV1} and \eqref{eq:TV2} in Section~\ref{sec:lower}.


\paragraph{Main Results.}

In this section, we present our results concerning the  thresholds for admissibility and impossibility of weak and strong detection, in different asymptotic regimes: (a) both $n$ and $d$ tends to infinity, (b) $n$ is a constant and $d$ tends to infinity, and (c) $d$ is a constant and $n$ tends to infinity. We begin with our upper-bounds. As was mentioned in the Introduction, \cite{nazer2022detecting} proposed the following simple test:
\begin{align}
    \phi_{\s{sum}}(\s{X},\s{Y})\triangleq \Ind\ppp{\s{sign}(\rho)\sum_{i,j=1}^nX_i^TY_j>|\rho|\frac{dn}{2}}.
\end{align}
We have the following result. 
\begin{theorem}{\cite[Theorem 1]{nazer2022detecting}}\label{th:UBbobak}
Consider the detection problem in \eqref{eqn:decprob}. Then,
\begin{align}
    \s{R}(\phi_{\s{sum}})\leq 2\cdot\exp\p{-\frac{d\rho^2}{60}}.\label{eqn:bobakrisk}
\end{align}
\end{theorem}
The implication of Theorem~\ref{th:UBbobak} is that if $\rho^2=\omega(d^{-1})$, then $\phi_{\s{sum}}$ achieves strong detection. Furthermore, if $\rho^2>\frac{60\log 2}{d}$, then $\phi_{\s{sum}}$ achieves weak detection. Note that the upper bound in \eqref{eqn:bobakrisk} is completely independent of $n$. We observe that in the case where $d$ is constant, the bound \eqref{eqn:bobakrisk} can never guarantee strong detection using $\phi_{\s{sum}}$, not even in the trivial case where $\rho^2=1$ (where detection with zero risk is possible). It should be emphasized that the above phenomenon is inherent; it can be shown that the boundary $\rho^2=\omega(d^{-1})$ associated with $\s{R}(\phi_{\s{sum}})$ cannot be improved and is not an artifact of the bounding technique used to establish \eqref{eqn:bobakrisk}. 

Aiming for strong detection in the scenario where $d$ is constant (and $\rho^2$ is smaller than unity), we propose the following alternative detection algorithm. For $i\in[n]$, let us define $\bar{X}_i \triangleq\frac{X_i}{\norm{X_i}_2}$ and $\bar{Y}_i \triangleq\frac{Y_i}{\norm{Y_i}_2}$. Then, define the test
\begin{align}
    \phi_{\s{count}}(\s{X},\s{Y})\triangleq\Ind\ppp{\sum_{i,j=1}^n\Ind\ppp{\s{sign}(\rho)\bar{X}_i^T\bar{Y}_j\geq|\rho|}\geq \frac{1}{2} n\calP_{d,\rho}},\label{eqn:testcount}
\end{align}
where $\calP_{d,\rho}\triangleq\pr_{\rho}\p{\bar{X}_1^T\bar{Y}_1\geq\rho}$, and
\begin{align}\label{eqn:p_rho}
    \pr_\rho \triangleq \mathcal{N}^{\otimes d}\p{\begin{bmatrix}
0 \\
0 
\end{bmatrix},\begin{bmatrix}
1 & \rho\\
\rho & 1
\end{bmatrix}}.
\end{align}
The following theorem, which is a result shows that as long as $\rho$ tends to one sufficiently fast, strong detection is possible.
\begin{theorem}\label{th:UBreconstruction}
Consider the detection problem in \eqref{eqn:decprob} and fix $2\leq d\in\mathbb{N}$. Then, $\s{R}(\phi_{\s{count}})\to 0$, as $n\to\infty$, if $\rho^2 = 1-o(n^{-\frac{2}{d-1}})$.
\end{theorem}
Prior to our work, for the case where $d$ is fixed and independent of $n$, it was not clear if strong detection (vanishing error probabilities) can be even achieved. In Theorem, it is shown for the first time that strong detection is possible under non-trivial conditions (i.e, $\rho^2<1$). We also point out the fact that our analysis hols only whenever $d\geq 2$. The case where $d=1$, on the other hand, remains a mystery. It turns out that many other alternative tests fail when $d=1$ as well. We suspect that this phenomenon might be inherent to the probabilistic structure whenever $d=1$, rather then an artifact of our algorithms. We leave this intriguing question open for future work.

\begin{rmk}[Recovery vs. detection] As mentioned in the introduction, the recovery problem of the permutation $\sigma$ was considered in \cite{pmlr-v89-dai19b}. It was shown that in the case where $d$ is constant, recovery is possible via the maximum-likelihood estimator if $\rho^2=1-o(n^{-\frac{4}{d}})$, while recovery is impossible if $\rho^2=1-\omega(n^{-\frac{4}{d}})$. Thus, Theorem~\ref{th:UBreconstruction} above, shows that the detection problem is statistically easier than recovery even when $d$ is fixed. In fact, denoting the maximum-likelihood estimator by $\hat{\sigma}^{\s{ML}}$, consider the following test
\begin{align}
    \phi_{\s{max}}(\s{X},\s{Y})\triangleq \Ind\ppp{\s{sign}(\rho)\sum_{i=1}^nX_i^TY_{\hat{\sigma}^{\s{ML}}_i}>|\rho|\frac{dn}{2}}.
\end{align}
It was claimed in \cite{nazer2022detecting} that as a corollary of \cite{pmlr-v89-dai19b}, this test achieves strong detection under the same recovery guarantee, namely, if $\rho^2=1-o(n^{-\frac{4}{d}})$. However, we suspect that this claim is in fact imprecise, as the authors overlooked the analysis of the Type-I error probability. Furthermore, it should be mentioned that \cite{tamir} proposed and analyzed a similar test as in \eqref{eqn:testcount}. It was shown that detection is possible if $\frac{\rho^2}{4} = 1-o(n^{-\frac{4}{d}})$. However, this condition is clearly meaningless since $|\rho|\leq 1$. Finally, we also mention that in other regimes, the threshold for the recovery problem has a significantly different behaviour compared to the detection problem; for example, if $d=\omega(\log n)$, then the recovery barrier is $\frac{\log n}{d}$ , i.e., recovery is possible (impossible) if $\rho^2\gg \frac{\log n}{d}$ ($\rho^2\ll \frac{\log n}{d}$). For detection, on the other hand, the barrier is $\frac{1}{d}$, independently of $n$.
\end{rmk}

In Theorem~\ref{th:LB}, we provide lower bounds, establishing thresholds for which weak detection is impossible.

\begin{theorem}[Impossibility of weak detection]\label{th:LB}
 Weak detection is impossible as long as $\rho^2=o(d^{-1})$. That is, for a sequence $(\rho,d,n)=(\rho_k,d_k,n_k)_k$ such that $\rho^2=o(d^{-1})$:
\begin{itemize}
     \item If $d$ is any function of $k$, and $n\to \infty$ then $\lim_{k\to \infty}\s{R}^\star =1$.
    \item If $n$ is constant and $d\to \infty$ then $\lim_{k\to \infty}\s{R}^\star =1$. 
\end{itemize}
Namely, $(\rho,d,n)$ is not admissible for  weak detection.
\end{theorem}

Our second lower bound concerns with impossibility of strong detection. In the case where $d$ is constant. In Theorem~\ref{th:LB2} we prove that for any $d$ there exists $\rho^\star=\rho^\star(d) $ such that for $\rho^2<\rho^\star$, strong detection is impossible. In the remaining cases, we show that the function $d^{-1}$ is a threshold for strong detection.

\begin{theorem}[Impossibility of strong detection]\label{th:LB2}
A sequence $(\rho,d,n)$ is not admissible for strong detection at either of the following scenarios: 
\begin{enumerate}
    \item $d\in \N$ and $\rho\in(-1,1)$ are constants such that $d<\frac{\log(\rho^2)}{\log(1-\rho^2)}$,  and $n\to\infty$.
    \item $n,d\to \infty$,  and $\rho^2<(1-\varepsilon)d^{-1}$ for some fixed  $\varepsilon>0$ which does not depend on $n$ and $d$.
    \item $d\to \infty$, $n$ is constant, and $\rho^2=O(d^{-1})$.
\end{enumerate} 
\end{theorem}
We remark that the function $d^
\star(\rho^2)=\frac{\log(\rho^2)}{\log(1-\rho^2)}$ is invertible as a function $(0,1)\to (0,\infty)$. Denoting the inverse function by $\rho^\star$, the condition established in Theorem~\ref{th:LB2} is equivalent to $\rho^2< \rho^\star(d)$.


\section{Upper Bounds}\label{sec:upper0}

Without loss of generality we assume below that $\rho>0$. Recall that for $i\in[n]$, we define $\bar{X}_i \triangleq\frac{X_i}{\norm{X_i}_2}$ and $\bar{Y}_i \triangleq\frac{Y_i}{\norm{Y_i}_2}$. Furthermore, let $\calQ_{d,\rho}\triangleq\pr_0\p{\bar{X}_1^T\bar{Y}_1\geq\rho}$, and $\calP_{d,\rho}\triangleq\pr_{\rho}\p{\bar{X}_1^T\bar{Y}_1\geq\rho}$, where $\pr_\rho$ is defined in \eqref{eqn:p_rho}, and
$\pr_0\equiv\pr_{\rho=0}$. Consider the test in \eqref{eqn:testcount}.
We start by bounding the Type-I error probability. Markov's inequality implies that
\begin{align}
    \pr_{\calH_0}\p{\phi_{\s{count}}=1} &= \pr_{\calH_0}\p{\sum_{i,j=1}^n\Ind\ppp{\bar{X}_i^T\bar{Y}_j\geq\rho}\geq \frac{1}{2} n\calP_{d,\rho}}\\
    &\leq \frac{2n^2\calQ_{d,\rho}}{n\calP_{d,\rho}}.\label{eqn:type1count}
\end{align}
On the other hand, we bound the Type-II error probability as follows. Under $\calH_1$, since our proposed test is invariant to reordering of $\s{X}$ and $\s{Y}$, we may assume without loss of generality that the latent permutation is the identity one, i.e., $\sigma=\s{Id}$. Then, Chebyshev's inequality implies that
\begin{align}
     \pr_{\calH_1}\p{\phi_{\s{count}}=0} &= \pr_{\calH_1}\p{\sum_{i,j=1}^n\Ind\ppp{\bar{X}_i^T\bar{Y}_j\geq\rho}< \frac{1}{2} n\calP_{d,\rho}}\\
     &\leq \pr_{\calH_1}\p{\sum_{i=1}^n\Ind\ppp{\bar{X}_i^T\bar{Y}_i\geq\rho}< \frac{1}{2} n\calP_{d,\rho}}\\
&\leq\frac{4\cdot\s{Var}_{\rho}\p{\sum_{i=1}^n\Ind\ppp{\bar{X}_i^T\bar{Y}_i\geq\rho}}}{n^2\calP^2_{d,\rho}},
\end{align}
where $\s{Var}_{\rho}$ denotes the variance with respect to $\P_\rho$, which  for the random variable $\sum_{i=1}^n\Ind\ppp{\bar{X}_i^T\bar{Y}_i}$ equals to the variance under the hypothesis $\calH_1$. Noticing that 
\begin{align}
    \s{Var}_\rho\p{\sum_{i=1}^n\Ind\ppp{\bar{X}_i^T\bar{Y}_i\geq\rho}} &= \sum_{i=1}^n\s{Var}_\rho\p{\Ind\ppp{\bar{X}_i^T\bar{Y}_i\geq\rho}}\\
    & = n\calP_{d,\rho}(1-\calP_{d,\rho}),
\end{align}
we finally obtain,
\begin{align}
     \pr_{\calH_1}\p{\phi_{\s{count}}=0} &\leq\frac{4(1-\calP_{d,\rho})}{n\calP_{d,\rho}}\leq\frac{4}{n\calP_{d,\rho}}.\label{eqn:type2count}
\end{align}
Next, we derive bounds on $\calQ_{d,\rho}$ and $\calP_{d,\rho}$. We start with a lower bound on $\calP_{d,\rho}$. Let $Z_1,\dots ,Z_d$ be i.i.d $\calN(0,I_{d\times d})$ random vectors, independent with $X_1,\dots,X_n$. Let $Y_i'\triangleq \rho X_i+\sqrt{1-\rho^2}Z_i$. We note that under $\P_\rho$, $(\s{X},\s{Y})$ is  equally distributed as $(\s{X},\s{Y}')$. Thus, for our analysis, we shall assume without loss of generality that $\s{Y}=\s{Y'}$. Under this assumption, we have
\begin{align}
    |\bar{X}_1^T\bar{Y}_1|^2 &= \frac{\rho^2\norm{X_1}_2^4+2\rho\sqrt{1-\rho^2}\norm{X_1}_2^2X_1^TZ_1+(1-\rho^2)|X_1^TZ_1|^2}{\rho^2\norm{X_1}_2^4+2\rho\sqrt{1-\rho^2}\norm{X_1}_2^2X_1^TZ_1+(1-\rho^2)\norm{X_1}_2^2\norm{Z_1}_2^2}\\
    &\geq\frac{\rho^2\norm{X_1}_2^4+2\rho\sqrt{1-\rho^2}\norm{X_1}_2^2X_1^TZ_1}{\rho^2\norm{X_1}_2^4+2\rho\sqrt{1-\rho^2}\norm{X_1}_2^2X_1^TZ_1+(1-\rho^2)\norm{X_1}_2^2\norm{Z_1}_2^2}\\
    & = \frac{\rho^2\norm{X_1}_2^2+2\rho\sqrt{1-\rho^2}\norm{X_1}_2\norm{Z_1}_2\cos(\varphi_{X_1Z_1})}{\rho^2\norm{X_1}_2^2+2\rho\sqrt{1-\rho^2}\norm{X_1}_2\norm{Z_1}_2\cos(\varphi_{X_1Z_1})+(1-\rho^2)\norm{Z_1}_2^2},
\end{align}
where we have used the fact that for $X_1,Z_1\in\mathbb{R}^d$, the inner product $X_1^TZ_1$ can be represented as $X_1^TZ_1 = \norm{X_1}_2\norm{Z_1}_2\cos\varphi_{X_1Z_1}$, where $\varphi_{X_1Z_1}$ denotes the angle between $X_1$ and $Z_1$. Also, note that $\varphi_{X_1Z_1}$ is statistically independent of $\norm{X_1}_2$ and $\norm{Z_1}_2$. Thus, straightforward algebra steps reveal that, 
\begin{align}
    \calP_{d,\rho} &= \pr_{\rho}\p{\bar{X}_1^T\bar{Y}_1\geq\rho}\\
    &\geq \pr_0\pp{\cos(\varphi_{X_1Z_1})\geq\frac{\rho}{2\sqrt{1-\rho^2}}\p{\frac{\norm{Z_1}_2}{\norm{X_1}_2}-\frac{\norm{X_1}_2}{\norm{Z_1}_2}}}\\
    &\geq \pr_0\pp{\cos(\varphi_{X_1Z_1})\geq\frac{\rho}{2\sqrt{1-\rho^2}}\p{\frac{\norm{Z_1}_2}{\norm{X_1}_2}-\frac{\norm{X_1}_2}{\norm{Z_1}_2}},\norm{X_1}_2\geq\norm{Z_1}_2}\\
    &\geq \pr_0\p{\cos(\varphi_{X_1Z_1})\geq0,\norm{X_1}_2\geq\norm{Z_1}_2}\\
    & = \pr_0\p{\cos(\varphi_{X_1Z_1})\geq0}\cdot\pr_0\p{\norm{X_1}_2\geq\norm{Z_1}_2}\\
    & = 1/4.\label{eqn:ppprob}
\end{align}
We now derive an upper bound on $\calQ_{d,\rho}$. For a fixed vector $y_1\in\mathbb{S}^{d-1}$ on the $d$-dimensional sphere, let us define $\calB_\rho(y_1) \triangleq \ppp{x_1\in\mathbb{S}^{d-1}:x_1^Ty_1\geq\rho}$. Since under $\P_0$, $\bar{X}_1$ and $\bar{Y}_1$ are independently and uniformly distributed on $\mathbb{S}^{d-1}$ we have
\begin{align}\label{eqn:qqprob}
    \calQ_{d,\rho} = \intop_{\mathbb{S}^{d-1}} \frac{\s{Vol}(\calB_\rho(y_1))}{\s{Vol}(\mathbb{S}^{d-1})^2}\mathrm{d}y_1,
\end{align}
where the volume of a set $\calA\subset\mathbb{R}^d$ is defined as $\s{Vol}(\calA) = \intop_{\calA}\mathrm{d}x$, where the integration is with respect to Lebesgue's measure on the sphere. Let us bound $\s{Vol}(\calB_\rho)$ from above. To that end, define the set 
\begin{align}
    \tilde{\calB}_{\rho}(y_1)=\left\{ v\in \mathbb{R}^d: \frac{v}{\left\| v \right\|}\in B_{\rho}(y_1)    \right\} =\left\{ v\in \mathbb{R}^d : v^T y_1\geq \rho \left\| v \right\|    \right\}.
\end{align}
Let $f:(v_1,\dots,v_d)\to (\varphi_1,\dots,\varphi_{d-1},r) $ denote the spherical coordinates  transformation and let $P_{d-1}$ denote the projection of an element from $\R^d$ onto its first $d-1$ coordinates. We observe that since $\tilde{\calB}_{\rho}(y_1)$ is the cone defined by the spherical cap $\calB_\rho({y_1})$, it follows that,
\begin{equation}
    \label{eq:1}
     f(\tilde{\calB}_{\rho}(y_1))=P\parenv*{f(B_{\rho}(y_1))}\times [0,\infty].
\end{equation}
Finally, let $\mu$ denote the Gaussian measure of a multivariate random vector with mean $y_1$, and covariance matrix $\Sigma = \sigma^2I_{d\times d}$, where $\sigma^2\geq0$. Then,
\begin{align}
    1&=\intop_{x\in\mathbb{R}^d}\mu(\mathrm{d}x)\\
    & = \intop_{\mathbb{R}^d}\frac{1}{(2\pi \sigma^2)^{d/2}}e^{-\frac{1}{2\sigma^2}\left\| v-y_1 \right\|_2^2}dv\\
    &\geq \intop_{\tilde{\calB}_{\rho}(y_1)}\frac{1}{(2\pi \sigma^2)^{d/2}}e^{-\frac{1}{2\sigma^2}\left\| v-y_1 \right\|_2^2}dv\\
    &\overset{(a)}{\geq}\intop_{\tilde{\calB}_{\rho}(y_1)}\frac{1}{(2\pi \sigma^2)^{d/2}}e^{-\frac{1}{2\sigma^2}(\norm{v}^2-2\rho \norm{v}+1)}dv\\
    &\overset{(b)}{=} \frac{1}{(2\pi \sigma^2)^{d/2}}\intop_{f(\tilde{\calB}_\rho(y_1))} e^{-\frac{1}{2\sigma^2}(r^2-2\rho r+1)} \parenv*{r^{d-1}\prod_{i=1}^{d-1} \sin^{d-1-i}(\varphi_i)}d\varphi_1\cdots d\varphi_{d-1} dr\\
    &\overset{(c)}{=}\frac{1}{(2\pi \sigma^2)^{d/2}}\underset{\vol(\calB_\rho(y_1))}{\underbrace{\intop_{P\parenv*{f(\calB_\rho(y_1)}}\parenv*{\prod_{i=1}^{d-1} \sin^{d-1-i}(\varphi_i)}d\varphi_1\cdots d\varphi_{d-1}}} \cdot\intop_{0}^{\infty}r^{d-1}e^{-\frac{1}{2\sigma^2}(r^2-2\rho r+1)}dr\\
    &\geq \frac{1}{(2\pi \sigma^2)^{d/2}} \vol(\calB_\rho(y_1))\intop_\rho^{\infty}\rho^{d-1}e^{-\frac{1}{2\sigma^2}(r^2-2\rho r+1)}dr\\
    &=\frac{1}{(2\pi \sigma^2)^{(d-1)/2}} \vol(\calB_\rho(y_1))e^{-\frac{1-\rho^2}{2\sigma^2}}\rho^{d-1} \intop_\rho^{\infty}\frac{1}{(2\pi \sigma^2)^{1/2}} e^{-\frac{(r-\rho)^2}{2\sigma^2}} dr\\
    &=\frac{\rho^{d-1}}{(2\pi \sigma^2)^{(d-1)/2}}e^{-\frac{1-\rho^2}{2\sigma^2}} \vol(\calB_\rho(y_1))\frac{1}{2},\label{eqn:lowerBoundGau}
\end{align}
where $(a)$ follows from the definition of $\tilde{\calB}_\rho(y_1)$ 
, $(b)$ follows from change of variables and $(c)$ follows from \eqref{eq:1}.
Thus, for every $y_1\in \mathbb{S}^{d-1}$,
\begin{align}
    \s{Vol}(\calB_\rho(y_1))&\leq \min_{\sigma^2\geq0}2e^{\frac{1-\rho^2}{2\sigma^2}+\frac{d-1}{2}\log(2\pi\sigma^2)}\rho^{1-d}\\
    & = 2e^{\frac{d-1}{2}\log\p{2\pi e\frac{1-\rho^2}{d-1}}}\rho^{1-d}.\label{eqn:qqup}
\end{align}
On the other hand, it is well-known that,
\begin{align}
    \s{Vol}(\mathbb{S}^{d-1}) = \frac{2\pi^{d/2}}{\Gamma(d/2)}\geq \frac{2\pi^{d/2}}{(d/2)^{d/2-1}} = \frac{4}{d}e^{\frac{d}{2}\log(2\pi/d)},\label{eqn:qqlow}
\end{align}
where we have used the fact that $\Gamma(x)<x^{x-1}$, for $x>1$ (see, e.g., \cite{anderson1997monotoneity}). Combining \eqref{eqn:qqlow}, \eqref{eqn:qqup}, and \eqref{eqn:qqprob}, we obtain
\begin{align}
    \calQ_{d,\rho}\leq f(d)e^{\frac{d-1}{2}\log(1-\rho^2)}\rho^{1-d},\label{eqn:qlowfinal}
\end{align}
where $f(d)\triangleq \frac{d}{2}\p{\frac{ed}{d-1}}^{d/2}\p{\frac{2\pi e}{d-1}}^{1/2}$. Finally, using \eqref{eqn:ppprob} and \eqref{eqn:qlowfinal}, we see that \eqref{eqn:type1count} can be further upper bounded as 
\begin{align}
    \pr_{\calH_0}\p{\phi_{\s{count}}=1} \leq 8nf(d)e^{\frac{d-1}{2}\log(1-\rho^2)}\rho^{1-d} = 8f(d)n(\rho^{-2}-1)^{\frac{d-1}{2}},
\end{align}
while \eqref{eqn:type2count} can be upper bounded as,
\begin{align}
\pr_{\calH_1}\p{\phi_{\s{count}}=0} \leq\frac{16}{n}.
\end{align}
Thus, for a fixed $d\geq2$, it is clear that $\pr_{\calH_1}\p{\phi_{\s{count}}=0}\to0$, as $n\to\infty$, and $\pr_{\calH_0}\p{\phi_{\s{count}}=1}\to0$, if $n(\rho^{-2}-1)^{\frac{d-1}{2}}=o(1)$. The later holds if $\rho^{-2} = 1+o(n^{-2/(d-1)})$, which implies that $\rho^{2} = 1-o(n^{-2/(d-1)})$, as stated.

\section{Lower Bounds}\label{sec:lower}

As in many detection problems, evaluating the minimax risk function  opposes a great challenge due to the error term obtained by maximizing over the error for all permutations in $\S_n$. A well known strategy for overcoming this inherent obstacle is by considering the softer average-case version of the problem. Let $\pi$ be the uniform measure on $\S_n$, and let us denote by $\P_{\calH_1}$ the probability measure obtained by averaging $\P_{\calH_1\vert\pi}$ with respect to $\pi$. For a test $\phi$, we consider the Bayesian risk function given by 
\begin{align}
\bar{\s{R}}(\phi) \triangleq \P_{\calH_0}[\phi(\s{X},\s{Y})=1]+\E_{\sigma\sim\pi}\pp{\P_{\calH_1\vert\sigma}[\phi(\s{X},\s{Y})=0]},
\end{align}
and the Bayesian risk for our problem:
\begin{align}
\bar{\s{R}}^\star\triangleq \inf_{\phi}\bar{\s{R}}(\phi).
\end{align}

Clearly, any test $\phi$ satisfies $\s{R}(\phi)\geq \bar{\s{R}}(\phi)$, and therefore $\s{R}^\star\geq \bar{\s{R}}^\star$. We conclude that in order to prove Theorem~\ref{th:LB}, it is sufficient to show that under the given assumptions, $\bar{\s{R}}^\star=1+o(1)$. Using a well-known equivalent characterization of the Bayesian risk function by the total variation distance and Cauchy-Schwartz inequality one shows that
\begin{equation}
    \label{eq:LikeBound}
    {\s{R}}^\star\geq \bar{\s{R}}^\star= 1-d_{\s{TV}}(\P_{\calH_0},\P_{\calH_1})\geq 1-\frac{1}{2}\sqrt{\E_{0}\pp{L^2}-1},
\end{equation}
where $L\triangleq\frac{\P_{\calH_1}}{\P_{\calH_0}}$ is the likelihood ratio, and the expectation is taken with respect to $\P_{\calH_0}$. Using the bound given in \eqref{eq:LikeBound}, it is sufficient to show that under the assumptions of Theorem~\ref{th:LB}, $\E_{0}\pp{L^2}\leq 1+o(1)$. 

Inspired by \cite{wu2020testing}, Zeynep and  Nazer  \cite{nazer2022detecting} gave an exact description of $\E_{0}\pp{L^2}$ using the distribution of cycles in a uniformly drawn  random permeation. In order to prove Theorem~\ref{th:LB}, we shall carefully analyse $\E_0[L^2]$, and improve the bounds proved \cite{nazer2022detecting}. For completeness of the paper, we outline the main ideas behind Nazer and Zeynep's calculation of $\E_{0}\pp{L^2}$ before we proceed toward our refined analysis. 

The first step in the calculation calls for a use of Ingster-Suslina method, stating that by Fubini's theorem, $\E_0[L^2]$ may be equivalently written as
\begin{equation}
    \label{eq:LikeSquare}\E_0[L^2]=\E_{\pi\indep\pi'}\pp{\E_0\pp{\frac{\P_{\calH_1|\pi}}{\P_{\calH_0}}\cdot \frac{\P_{\calH_1|\pi'}}{\P_{\calH_0}}}},
\end{equation}
where the expectation is taken with respect to the independent coupling of $\pi$ and $\pi'$, two copies of the uniform measure on $\S_n$. For  fixed permutations $\sigma$ and $\sigma'$, we note that $\P_{\calH_1\vert\sigma},\P_{\calH_1|\sigma'}$ and $\P_{0}$ are absolutely continuous with respect to Lebesgue's measure on $\R^{2\times d\times n}$ and therefore we have
\begin{align}
\frac{\P_{\calH_1\vert\sigma}}{\P_{\calH_0}}=\frac{f_{\calH_1|\sigma}}{f_{\calH_0}}, \quad\text{and} \quad \frac{\P_{\calH_1|\sigma'}}{\P_{\calH_0}}=\frac{f_{\calH_1|\sigma'}}{f_{\calH_0}}, 
\end{align}
where $f_{\calH_i}$ denotes the Radon-Nikodym derivative of $\P_{\calH_i}$ with respect to Lebesgue's measure, which is the density function $(\s{X},\s{Y})$ under the corresponding hypothesis. Let $\calN_{\rho}:\R^{2\times d}:\to \R_+$ denote the density function of a pair of random vectors distributed as 
\[\calN^{\otimes d} \p{\begin{bmatrix}0\\0 \end{bmatrix}, \begin{bmatrix}1 &\rho \\ \rho &1 \end{bmatrix} }.\]
We note that 
\begin{equation}
    \label{eq:LikeRatDis}
      \frac{f_{\calH_1|\sigma}(\s{X},\s{Y})}{f_{\calH_0}(\s{X},\s{Y})}\frac{f_{\calH_1|\sigma'}(\s{X},\s{Y})}{f_{\calH_0}(\s{X},\s{Y})}=\prod_{i=1}^n \frac{\calN_{\rho}(X_i,Y_{\sigma_i})}{\calN_0(X_i,Y_{\sigma_i})} \frac{\calN_{\rho}(X_i,Y_{\sigma'_i})}{\calN_0(X_i,Y_{\sigma'_i})}. 
\end{equation}

In order to proceed with the calculation, we make two key observations. First, we note that the distribution under the null hypothesis ($\calH_0$) is invariant to reordering the coordinates. In a similar manner, the uniform measure on $\S_n$, is invariant under composition with a fixed permutation. Thus, 
\begin{align}\label{eq:Idper}
    \E_{\pi\indep\pi'}\pp{\E_0\pp{\frac{\P_{\calH_1|\pi}}{\P_{\calH_0}}\cdot \frac{\P_{\calH_1|\pi'}}{\P_{\calH_0}}}}=\E_{\pi}\pp{\E_0\pp{\frac{\P_{\calH_1|\pi}}{\P_{\calH_0}}\cdot \frac{\P_{\calH_1|\Id}}{\P_{\calH_0}}}}.
\end{align}

We consider the product given in \eqref{eq:LikeRatDis} for a fixed $\sigma\in \S_n$ and $\sigma'=\Id$, which we denote by $Z_{\sigma}$. The second key observation, is that $Z_{\sigma}$ can be decomposed to independent terms, corresponding to the cycles of the permeation $\sigma$. We recall that a cycle of a permutation $\sigma$ is a string $(i_0,i_2,\dots,i_{|C|-1})$ of elements in $[n]$ such that $\sigma(i_j)=i_{j+1\mod|C|}$  for all $j$. If $\abs{C}=k$, we call $C$ a $k$-cycle. For a fixed cycle $C$, we denote
\begin{align}
Z_C\triangleq \prod_{i\in C}\frac{\calN_{\rho}(X_i,Y_{\sigma_i})}{\calN_0(X_i,Y_{\sigma_i})} \frac{\calN_{\rho}(X_i,Y_{i})}{\calN_0(X_i,Y_{i})}.
\end{align}
Since the set of cycles of a permutation induce a partition of $[n]$, the random variables $\{Z_C\}_C$, corresponding to all cycles of $\sigma$, are independent (with respect to $\P_{\calH_0}$) and 
\begin{equation}
    \label{eq:CycleProd}
    Z_\sigma=\prod_{C}Z_C.
\end{equation} 

The following lemma states that for a fixed cycle $C$, $\E_0[Z_C]$ depends on $\rho$ and $\abs{C}$. The proof of the lemma is based on the properties of Gaussian random vectors. For further details the reader is referred to \cite[Lemma 10]{nazer2022detecting}.
\begin{lemma}\label{lem:CycleComp}
For a fixed cycle $C$ of a permutation $\sigma$, 
\begin{align}
\E_0[Z_C]=\frac{1}{(1-\rho^{2|C|})^{d}}.
\end{align}
\end{lemma}

For a fixed permutation $\sigma\in \S_n$ and $k\in [n]$, let $N_k(\sigma)$ denote the number of $k$-cycles of $\sigma$. Combining \eqref{eq:LikeSquare}, \eqref{eq:Idper}, \eqref{eq:CycleProd}, and Lemma~\ref{lem:CycleComp} we obtain
\begin{equation}
    \label{eq:PermProd}
    \E_0[L^2]=\E_{\pi}\pp{\prod_{C}Z_C}=\E_{\pi}\pp{\prod_{k=1}^n\frac{1}{(1-\rho^{2k})^{dN_k}}}.
\end{equation}

By analysis of \eqref{eq:PermProd}, Zeynep and Nazer showed in \cite[Lemma 3]{nazer2022detecting} that $E_0[L^2]\leq (1-\rho^2)^{-dn}$, which equals $1+o(1)$ if $\rho^2=o((nd)^{-1})$. Inspired by the calculation performed in \cite[Proposition 2]{wu2020testing}, we carefully bound \eqref{eq:PermProd} from above, utilizing the statistical properties of $k$-cycles in a uniformly distributed random permutation. Our refined analysis enables us to prove that $\E_0[L^2]\leq 1+o(1)$ assuming only that $\rho^2=o(d^{-1})$. The following proposition is makes the main argument for the proof of our lower bounds given in Theorem~\ref{th:LB} and Theorem~\ref{th:LB2}.

\begin{prop}\label{prop:MainLB}
   Let $N_k$ be the number of $k$-cycles in a uniformly distributed permutation $\pi$, and $1\leq k \leq n$. Then: 
   \begin{enumerate}
   \item For all $(\rho^2,d,n)$ is holds that
       \begin{equation}
        \label{eq:bound-Dconst2}
        \E_{\pi}\pp{\prod_{k=1}^n\p{\frac{1}{1-\rho^{2k}}}^{dN_k}}\leq\exp\p{\frac{nd\rho^2}{1-\rho^2}}.
    \end{equation}
       \item If  at least one of $n,d$ tends to $\infty$ and $\rho^2=o(d^{-1})$, then
    \begin{equation}
        \label{eq:ProdBound} \E_{\pi}\pp{\prod_{k=1}^n\p{\frac{1}{1-\rho^{2k}}}^{dN_k}}\leq 1+o(1).
    \end{equation}
    \item If both $n$, $d$ tends to $\infty$ and $\rho^2<(1-\varepsilon)d^{-1}$ for some $\varepsilon>0$ then 
       \begin{equation}
        \label{eq:bound-Dconst1}
        \E_{\pi}\pp{\prod_{k=1}^n\p{\frac{1}{1-\rho^{2k}}}^{dN_k}}\leq (1+o(1)) \exp\p{\frac{d\rho^2}{1-\rho^2}+\frac{c(d,\rho^2)\rho^4}{1-\rho^4}}.
    \end{equation}
    \item If $d$ and $\rho^2$ are a constant  satisfying  
    $d<\frac{\log(\rho^2)}{\log(1-\rho^2)}$ and $n\to \infty$
    \begin{equation}
        \label{eq:bound-Dconst}
        \E_{\pi}\pp{\prod_{k=1}^n\p{\frac{1}{1-\rho^{2k}}}^{dN_k}}\leq (1+o(1)) \exp\p{\frac{d\rho^2}{1-\rho^2}+\frac{c(d,\rho^2)\rho^4}{1-\rho^4}},
    \end{equation}
   \end{enumerate}
   where $c(d,\rho^2)=\frac{d(d+1)}{2(1-\rho^2)^{d+2}}$.
\end{prop}

For the proof of this proposition, we shall require several technical results. The following lemma concerns the approximation of the joint distribution of $k$-cycles by independent Poisson random variables.

\begin{lemma}\label{lem:PoisApp}\cite[Theorem 2]{arratia1992cycle} Let $1\leq k\leq n$ be an integer, and let $Z_1\dots,Z_k$ be independent random variables such that for all $1\leq i \leq k$, $Z_i\sim \s{Poisson}\p{i^{-1}}$. Then, the total variation between the law of $N_1,\dots, N_k$ and $Z_1,\dots,Z_k$ satisfies
\begin{align}
d_{\s{TV}}\p{\calL\p{N_1,N_2,\ldots,N_k},\calL\p{Z_1,Z_2,\ldots,Z_k}}\leq F\p{\frac{n}{k}},
\end{align}
where $F(x)$ is a monotone decreasing function satisfying $\log F(x)=-x\log x
(1+o(1))$ as $x\to \infty$.

\end{lemma}

\begin{lemma}\label{lem:PoisExp}
Let $1\leq m\leq n$ be an integer, and let $Z_1\dots,Z_m$ be independent random variables such that for all $1\leq i \leq m$, $Z_i\sim \s{Poisson}\p{i^{-1}}$. Then, 
\begin{align}
\E_{\pi}\pp{\prod_{k=1}^m\p{\frac{1}{1-\rho^{2k}}}^{dZ_k}}\leq \exp\p{\frac{d\rho^2}{1-\rho^2}+\frac{c(d,\rho^2)\rho^4}{1-\rho^4}},
\end{align}
 where $c(d,\rho^2)=\frac{d(d+1)}{2(1-\rho^2)^{d+2}}$, and therefore if $\rho^2=o(d^{-1}),$
 \begin{align}
 \E_{\pi}\pp{\prod_{k=1}^m\p{\frac{1}{1-\rho^{2k}}}^{dZ_k}}\leq 1+o(1).
 \end{align}
\end{lemma}
\begin{proof}
The proof of this lemma is elementary, only requires the moment generating function of Poisson random variables, and some linear approximations of elementary functions. By rearranging the expression in the expectation and using independence we have
\begin{align}
    \E_{\pi}\pp{\prod_{k=1}^m\p{\frac{1}{1-\rho^{2k}}}^{dZ_k}}&=\prod_{k=1}^m\E_{\pi}\pp{\p{\frac{1}{1-\rho^{2k}}}^{dZ_k}}\\
    &=\prod_{k=1}^m\E_{\pi}\pp{\exp\p{-d Z_k \log\p{1-\rho^{2k}}}}\\
    &\overset{(a)}{=}\prod_{k=1}^m\exp\p{\frac{1}{k}\p{e^{-d\log(1-\rho^{2k})}-1}}\\
    &=\prod_{k=1}^m\exp\p{\frac{1}{k}\p{\frac{1}{(1-\rho^{2k})^d}-1}}.
\end{align}
where (a) is followed by the definition of the moment generating function of a Poisson random variable.

We shall now bound the term $(1-\rho^{2k})^{-d}-1$ from above. A straight forward calculation of the Taylor expansion of the function $f(x)=(1-x)^{-d}$ show that for $x\in(0,1)$ we have 
    \[\frac{1}{(1-x)^d}=\sum_{m=0}^\infty \binom{m+d-1}{d-1}x^m.\]
Using Lagrange's remainder theorem and obtain that for all $x>(0,1)$,
\begin{align}
    \frac{1}{(1-x)^d}&=1+dx+\frac{d(d+1)}{2}\cdot\frac{1}{(1-c)^{(d+2)}}x^2\\
    &\leq 1+dx+\frac{d(d+1)}{2(1-x)^{(d+2)}}x^2.
\end{align}
where $c$ is a point in $[0,x]$. Choosing $x=\rho^{2k}$ we get that for any $k>0$
\begin{align}
    \frac{1}{(1-\rho^{2k})^d}&\leq 1+d\rho^{2k}+\frac{d(d+1)}{2(1-\rho^{2k})^{(d+2)}}\rho^{4k}\\
    & \leq 1+d\rho^{2k}+\frac{d(d+1)}{2(1-\rho^{2})^{(d+2)}}\rho^{4k}.
\end{align}
For the rest of our analysis we denote 
\begin{align}
c(d,\rho^2)\triangleq \frac{d(d+1)}{2(1-\rho^{2})^{(d+2)}}\rho^{4k},
\end{align}
and we get
\begin{align}
       \E_{\pi}\pp{\prod_{k=1}^m\p{\frac{1}{1-\rho^{2k}}}^{dZ_k}}&=\prod_{k=1}^m\E_{\pi}\pp{\p{\frac{1}{1-\rho^{2k}}}^{dZ_k}}\\
    &=\prod_{k=1}^m\exp\p{\frac{1}{k}\p{\frac{1}{(1-\rho^{2k})^d}-1}}\\
    &\leq \prod_{k=1}^m\exp\p{\frac{1}{k}\p{1+\p{d\rho^{2k}+c(d,\rho^2)\rho^{4k}}-1}}\\
    &=\exp\p{d\sum_{k=1}^m\frac{\rho^{2k}}{k}+c(d,\rho^2)\sum_{k=1}^m\frac{\rho^{4k}}{k}}\\
    &\leq\exp\p{d\sum_{k=1}^\infty\frac{\rho^{2k}}{k}+c(d,\rho^2)\sum_{k=1}^\infty\frac{\rho^{4k}}{k}}\\
    &=\exp\p{-d\log(1-\rho^2)-c(d,\rho^2)\log(1-\rho^4)}\\
    &\overset{(a)}{\leq}
    \exp\p{\frac{d\rho^2}{1-\rho^2}+\frac{c(d,\rho^2)\rho^4}{1-\rho^4}},
\end{align}
where $(a)$ follows from the well-known inequality $\log(1+x)\geq x/(1+x)$, for $x>-1$. In the case where $\rho^2=o(d^{-1})$ clearly $d\rho^2=o(1)$. Furthermore,
\[c(d,\rho^2)\rho^4=\frac{d(d+1)}{2(1-\rho^2)^{d+2}}\rho^{4}\leq\frac{1}{2} \exp((d+2)\rho^2)d(d+1)\rho^4=\exp(o(1))\cdot o(1)=o(1).\]
In particular, we get 
\begin{align}
    \E_{\pi}\pp{\prod_{k=1}^m\p{\frac{1}{1-\rho^{2k}}}^{dZ_k}}&\leq  \exp\p{\frac{d\rho^2}{1-\rho^2}+\frac{c(d,\rho^2)\rho^4}{1-\rho^4}}\\
    &=\exp\p{\frac{o(1)}{1-o(1)}+\frac{o(1)}{1-o(1)}}=\exp(o(1))=   1+o(1).
\end{align}

\end{proof}

We are now ready to prove Proposition~\ref{prop:MainLB}. The idea of the proof is as follows: we consider the expectation of the product given in \eqref{eq:ProdBound}. In the case that $n\to \infty$, we show that the product of the last $n-\alpha\log n$ terms is always upper-bounded by $1+o(1)$ for an appropriate choice of $\alpha$. For the expectation of the product of the first $\alpha\log n $ terms, use the Poisson approximation of $\{N_k\}_k$ given in Lemma~\ref{lem:PoisApp} and the estimation in that case, given in Lemma~\ref{lem:PoisExp}. The other case, where $n$ is constant, is solvable using elementary arguments.

\begin{proof}[Proof of Proposition~\ref{prop:MainLB}]
We divide our proof into three parts, with respect to the asymptotic regimes of $n$ and $d$. We start  by a simple observation - whenever $n$ is fixed we have $\sum_{k=1}^n k N_K=n$, which implies that for any $1\leq m \leq n$ we have
\begin{align}
    \prod_{k=m}^n \p{\frac{1}{1-\rho^{2k}}}^{dN_k}&\leq\prod_{k=m}^n \p{\frac{1}{1-\rho^{2m}}}^{dN_k}=\p{\frac{1}{1-\rho^{2m}}}^{d\sum_{k=m}^nN_k}\\
    &\leq \p{\frac{1}{1-\rho^{2m}}}^{d n}= \p{1+\frac{\rho^{2m}}{1-\rho^{2m}}}^{d n}\leq \exp\p{\frac{dn\rho^{2m}}{1-\rho^{2m}}}.\label{eq:tailBoundGen}
\end{align}
This immediately proves \eqref{eq:bound-Dconst2}.

\paragraph{The case where both $n$ and $d$ tends to  $\infty$:} we assume that $n,d\to \infty$ and $\rho^2=o(d^{-1})$ or $\rho^2<(1-\varepsilon)d^{-1}$. 
We choose $m=\ceil{\log n}$ and we get
\begin{align}
    dn\rho^{2m}&\leq (d\rho^2) \cdot n(\rho^{2})^{\log(n)-1}=(d\rho^2)n(\rho^{2})^{\log\p{\frac{n}{e}}}=e(d\rho^2)\p{\frac{n}{e}}^{1+\log \rho^2}.
\end{align}
Since $\rho^2<d^{-1}$ (which is clearly true for sufficiently large $d$ in the particular case where $\rho^2=o(d^{-1})$), and $d\to \infty$, we have $\log \rho^2\to -\infty$ as $n\to \infty$. Plugging our chosen $m$ in \eqref{eq:tailBoundGen} obtain:
\begin{equation}\label{eq:LBtail}
    \prod_{k=\log n}^n \p{\frac{1}{1-\rho^{2k}}}^{dN_k} \leq \exp\p{\frac{dn\rho^{2m}}{1-\rho^{2m}}}=\exp(o(1))=1+o(1).
\end{equation}

For a fixed integer $m$, we consider the set $S_{n,m}\subseteq \N^m$ given by 
\begin{align}
S_{n,m}=\ppp{(n_1,\dots,n_m)\in \N^d; \sum_{k=1}^m n_k\leq n},
\end{align}
and a function $f_{n,m}:\N^m\to [0,\infty]$ given by 
\begin{align}
f_{n,m}(n_1,n_2,\dots,n_m)=\prod_{k=1}^m \p{\frac{1}{1-\rho^{2k}}}^{dn_k}\cdot \Ind_{S_{n,m}}(n_1,\dots,n_m).
\end{align}
We note that for all $n_1,\dots,n_m\in \N$,
\begin{equation}
    \label{eq:LBfBound}
    f_{n,m}(n_1,n_2,\dots,n_m)\leq \prod_{k=1}^m \p{\frac{1}{1-\rho^{2}}}^{dn_k}\cdot \Ind_{S_{n,m}}(n_1,\dots,n_m)\leq \p{\frac{1}{1-\rho^{2}}}^{dn}.
\end{equation}

We set $m=\ceil{\log n} $, and let $\{Z_k\}_k$ be independent $\s{Poisson}\p{k^{-1}}$ random variables as in Lemma~\ref{lem:PoisApp}.  Since $\sum_{k=1}^{m} N_k\leq n$ with probability $1$, we have 
\begin{align}
     \E_{\pi}\pp{\prod_{k=1}^m\p{\frac{1}{1-\rho^{2k}}}^{dN_k}}&=\E_{\pi}\pp{f_{n,m}(N_1,\dots,N_{m})}\\
     &\leq \E[f_{n,m}(Z_1,\dots,Z_m)]+d_{\s{TV}}\p{\calL\p{N_1^m},\calL\p{Z_1^m}}\cdot \norm{f_{n,m}}_\infty \\
     &\leq \E_{\pi}\pp{\prod_{k=1}^m\p{\frac{1}{1-\rho^{2k}}}^{dZ_k}}+d_{\s{TV}}\p{\calL\p{N_1^m},\calL\p{Z_1^m}}\cdot \norm{f_{n,m}}_\infty \\
     &\overset{(a)}{\leq} \E_{\pi}\pp{\prod_{k=1}^m\p{\frac{1}{1-\rho^{2k}}}^{dZ_k}}+F\p{\frac{n}{m}}\cdot \p{\frac{1}{1-\rho^{2}}}^{dn} \\
      &\overset{(b)}{\leq} \exp\p{\frac{d\rho^2}{1-\rho^2}+\frac{c(d,\rho^2)\rho^4}{1-\rho^4}}+F\p{\frac{n}{m}}\cdot \p{\frac{1}{1-\rho^{2}}}^{dn}\\
      &= \exp\p{\frac{d\rho^2}{1-\rho^2}+\frac{c(d,\rho^2)\rho^4}{1-\rho^4}}+F\p{\frac{n}{\ceil{\log n}}}\cdot \p{\frac{1}{1-\rho^{2}}}^{dn},
\end{align}
where (a) follows from \eqref{eq:LBfBound} and Lemma~\ref{lem:PoisApp}, (b) follows from Lemma~\ref{lem:PoisExp}.
By Lemma~\ref{lem:PoisApp},  we also have 
\begin{align}
    \log\p{F\p{\frac{n}{\ceil{\log n}}}\p{\frac{1}{1-\rho^{2}}}^{dn}}&\leq \log\p{F\p{\frac{n}{\log n}}\p{\frac{1}{1-\rho^{2}}}^{dn}}\\
    &=\log\p{F\p{\frac{n}{\log n}}}-nd\log\p{1-\rho^{2}}\\
    &\leq -(1+o(1))\frac{n}{\log n}\log\p{\frac{n}{\log n}}-nd\log\p{1-\rho^{2}}\\
    &\overset{(a)}{\leq} -(1+o(1))\frac{n}{\log n}\log\p{\frac{n}{\log n}}+nd\rho^{2}(1+o(1))\\
    &=n\p{-(1+o(1))\p{1-\frac{\log \log n }{\log n}}+d\rho^{2}(1+o(1))}\\
    &=-n(1-d\rho^2+o(1))\xrightarrow[n\to \infty]{(b)} -\infty,
\end{align}
Where $(a)$ follows from the Taylor expansion of the function $\log(1-x)$ and $\rho^2=o(1)$, and (b) follows from the assumption that $\rho^2<(1-\varepsilon)d^{-1}$. This implies that 
\begin{align}
F\p{\frac{n}{\ceil{\log n}}}\p{\frac{1}{1-\rho^{2}}}^{dn}=o(1),
\end{align}
 and therefore,
\begin{equation}\label{eq:LBprefBound}
    \E_{\pi}\pp{\prod_{k=1}^{\ceil{\log n}}\p{\frac{1}{1-\rho^{2k}}}^{dN_k}}\leq  \exp\p{\frac{d\rho^2}{1-\rho^2}+\frac{c(d,\rho^2)\rho^4}{1-\rho^4}}+o(1).
\end{equation}
Combining \eqref{eq:LBtail} and \eqref{eq:LBprefBound} together we conclude:
 \begin{align}
     \E_{\pi}\pp{\prod_{k=1}^{n}\p{\frac{1}{1-\rho^{2k}}}^{dN_k}}&\leq \E_{\pi}\pp{\prod_{k=1}^{\ceil{\log n}}\p{\frac{1}{1-\rho^{2k}}}^{dN_k}\prod_{k=\ceil{\log n}}^{n}\p{\frac{1}{1-\rho^{2k}}}^{dN_k}}\\
     &=(1+o(1)) \E_{\pi}\pp{\prod_{k=1}^{\ceil{\log n}}\p{\frac{1}{1-\rho^{2k}}}^{dN_k}}\\
     &=(1+o(1)) \exp\p{\frac{d\rho^2}{1-\rho^2}+\frac{c(d,\rho^2)\rho^4}{1-\rho^4}}+o(1).
 \end{align}
We have now proved \eqref{eq:bound-Dconst1}. We note that by assuming furthermore that $\rho^2=o(d^{-1})$, by the second part of Lemma~\ref{lem:PoisExp} we get
\begin{align}
    \E_{\pi}\pp{\prod_{k=1}^{n}\p{\frac{1}{1-\rho^{2k}}}^{dN_k}}&\leq(1+o(1))\exp\p{\frac{d\rho^2}{1-\rho^2}+\frac{c(d,\rho^2)\rho^4}{1-\rho^4}}+o(1)=1+o(1).
\end{align}
We have now proved \eqref{eq:ProdBound}.
\paragraph{The case where $n$ is constant and $d$ tends to  $\infty$:} We assume that $\rho^2=o(d^{-1})$. Since $n$ is constant, we have $dn\rho^2=o(1)$ and therefore by \eqref{eq:tailBoundGen} we have
\begin{align}
    E_{\pi}\pp{\prod_{k=1}^n \p{\frac{1}{1-\rho^{2k}}}^{dN_k}}&\leq \exp\p{\frac{dn\rho^{2}}{1-\rho^{2}}}=1+o(1).
\end{align}

\paragraph{The case where $d$ is constant and $n$ tends to  $\infty$:}  we also assume that $\rho^2$ is a constant such that $d< \frac{\log(\rho^2)}{\log(1-\rho^2)}$. We repeat the same steps as in the case where $n\to\infty$ and $\rho^2=o(d^{-1})$ with a minor change. Instead of approximating the product of the first $m=\ceil{\log n}$ terms in the product, we take $m'=\ceil{\alpha \log(n)}$, where $\alpha=-\frac{1}{\log(\rho^2)}+\varepsilon$, where $\varepsilon$ sufficiently small so that
\begin{equation}
    \label{eq:epsChoice}
    d\frac{\log(1-\rho^2)}{\log(\rho^2)}<\frac{1}{1-\varepsilon\log(\rho^2)},
\end{equation}
(such exists by the assumption $d<d^\star(\rho^2)=\frac{\log(\rho^2)}{\log(1-\rho^2)}$). Repeating the same steps as in the previous part, we have 
\begin{align}
\prod_{k=m'}^{n}\p{\frac{1}{1-\rho^{2k}}}^{dN_k}\leq \exp\p{\frac{dn\rho^{2m'}}{1-\rho^{2m'}}}. 
\end{align}
We observe that 
\begin{align}
    nd\rho^{2m'}&\leq nd(\rho^{2})^{\alpha \log n}=ndn^{\alpha \log(\rho^2)}=dn^{1+\alpha \log(\rho^2)}=dn^{\varepsilon \log(\rho^2)}=o(1),
\end{align}
which implies that 
\begin{align}
\prod_{k=m'}^{n}\p{\frac{1}{1-\rho^{2k}}}^{dN_k}\leq \exp\p{\frac{dn\rho^{2m'}}{1-\rho^{2m'}}}=1+o(1).
\end{align}

We now evaluate the product of the first $m'$ terms. In a similar fashion to the first part, by Lemma~\ref{lem:PoisApp} and Lemma~\ref{lem:PoisExp},
\begin{align}
     \E_{\pi}\pp{\prod_{k=1}^{m'}\p{\frac{1}{1-\rho^{2k}}}^{dN_k}}&\leq \E_{\pi}\pp{\prod_{k=1}^{m'}\p{\frac{1}{1-\rho^{2k}}}^{dZ_k}}+F\p{\frac{n}{m'}}\cdot \p{\frac{1}{1-\rho^{2}}}^{dn} \\
      &\leq \exp\p{\frac{d\rho^2}{1-\rho^2}+\frac{c(d,\rho^2)\rho^4}{1-\rho^4}}+F\p{\frac{n}{m'}}\cdot \p{\frac{1}{1-\rho^{2}}}^{dn}\\
      =& \exp\p{\frac{d\rho^2}{1-\rho^2}+\frac{c(d,\rho^2)\rho^4}{1-\rho^4}}+F\p{\frac{n}{\ceil{\alpha \log n }}}\cdot \p{\frac{1}{1-\rho^{2}}}^{dn}.
\end{align}
Using Lemma~\ref{lem:PoisApp} once again,  we obtain
\begin{align}
    &\log\p{F\p{\frac{n}{\ceil{\alpha\log n}}}\p{\frac{1}{1-\rho^{2}}}^{dn}}\leq  -(1+o(1))\frac{n}{\alpha\log n}\log\p{\frac{n}{\alpha\log n}}-nd\log(1-\rho^{2})\\
    &\hspace{2cm}=-n\p{(1+o(1))\frac{1}{\alpha}\p{1-\frac{\log (\alpha\log n) }{ \log n}}+d\log(1-\rho^{2})}\\
    &\hspace{2cm}=-n(1+o(1))\p{\frac{1}{\alpha}+d\log(1-\rho^{2})}\\
    &\hspace{2cm}=-n(1+o(1))\p{\frac{\log(\rho^2)}{\varepsilon\log(\rho^2)-1}+d\log(1-\rho^{2})}.
\end{align} 
By \eqref{eq:epsChoice} we have
\begin{align}
\frac{\log(\rho^2)}{\varepsilon\log(\rho^2)-1}+d\log(1-\rho^{2})>0,
\end{align}
which implies that 
\begin{align}
F\p{\frac{n}{m'}}\cdot \p{\frac{1}{1-\rho^{2}}}^{dn}=o(1).
\end{align}
We now conclude
\begin{align}
     \E_{\pi}\pp{\prod_{k=1}^{n}\p{\frac{1}{1-\rho^{2k}}}^{dN_k}}&= \E_{\pi}\pp{\prod_{k=1}^{m'}\p{\frac{1}{1-\rho^{2k}}}^{dN_k}\prod_{k=m'}^{n}\p{\frac{1}{1-\rho^{2k}}}^{dN_k}}\\
     &= \E_{\pi}\pp{\prod_{k=1}^{m'}\p{\frac{1}{1-\rho^{2k}}}^{dN_k}(1+o(1))}\\
      &\leq (1+o(1))\p{\exp\p{\frac{d\rho^2}{1-\rho^2}+\frac{c(d,\rho^2)\rho^4}{1-\rho^4}}+o(1)}\\
      &= (1+o(1))\exp\p{\frac{d\rho^2}{1-\rho^2}+\frac{c(d,\rho^2)\rho^4}{1-\rho^4}}.
\end{align}
\end{proof}

The proofs of Theorem~\ref{th:LB} and Theorem~\ref{th:LB2} now follows:

\begin{proof}[Proof of Theorem~\ref{th:LB}]
We combine \eqref{eq:LikeSquare}, \eqref{eq:PermProd} and Proposition~\ref{prop:MainLB} and obtain:
\begin{align}
    {\s{R}}^\star&\geq 1-\sqrt{\E_0[L^2]-1}\\
    &=1-\sqrt{\E_{\pi}\pp{\prod_{k=1}^n\frac{1}{(1-\rho^{2k})^{dN_k}}}-1}\\
    &\geq 1-\sqrt{1+o(1)-1}\\
    &=1+o(1).
\end{align}

\end{proof}

\begin{proof}[Proof of Theorem~\ref{th:LB2}]
Let $(\rho,d,n)$ be a sequence satisfying the assumptions of Theorem~\ref{th:LB2}. We start by recalling an important well-known fact (see, for example, \cite[Lemma 2.6 and 2.7]{tsybakov2004introduction}): for any sequence of measures $(\P_{0,k})_k,(\P_{1,k})_k$, 
\[\E_{\P_{0,k}}\pp{\p{\frac{\P_{1,k}}{\P_{0,k}}}^2}=O(1)\implies d_{\s{TV}}(\P_{0,k},\P_{1,k})=1-\Omega(1).\]
Thus, by \eqref{eq:LikeBound}, if $\E_0[L^2]=O(1)$ we have 
\begin{equation}\label{eq:ImpStrongDet}
    \s{R}^\star\geq 1-d_{\s{TV}}(\P_{\calH_0},\P_{\calH_1})=\Omega(1),
\end{equation}
which implies that $(\rho,d,n)$ is not admissible for strong detection (as \eqref{eq:ImpStrongDet} implies $\limsup\s{R}^\star>0$).

Indeed, by Proposition~\ref{prop:MainLB} and \eqref{eq:PermProd}: if $d$ and $\rho$ are constants such that $d<d^\star(\rho^2)$, and $n\to \infty$, then 
\[\E_0[L^2]\leq (1+o(1))\exp\p{\frac{d\rho^2}{1-\rho^2}+\frac{c(d,\rho^2)\rho^4}{1-\rho^4}}=O(1). \]

On the other hand, if $d,n\to\infty$ and $\rho^2<(1-\varepsilon)d^{-1}$, we have that $\rho^2=o(1)$. Thus, it also follows from Proposition~\ref{prop:MainLB} that 
\begin{align}
    \E_0[L^2]&\leq (1+o(1))\exp\p{\frac{d\rho^2}{1-\rho^2}+\frac{c(d,\rho^2)\rho^4}{1-\rho^4}}\\
    &\leq(1+o(1))\exp\p{(1+o(1))\p{d\rho^2+\frac{d(d+1)}{2(1-\rho^2)^{d+2}}\rho^4}}\\
    &\leq (1+o(1))\exp\p{(1+o(1))\p{1+\frac{1}{2e^{-2}(1+o(1))}}}=O(1).
\end{align}

The remaining case is where $d\to\infty$, $n$ is constant, and $\rho^2=O(d^{-1})$. By \eqref{eq:bound-Dconst2}, which is true without any assumptions on $(\rho,d,n)$, we have
\begin{align}
\E_0[L^2]\leq \exp\p{\frac{nd\rho^2}{1-\rho^2}}=O(1).
\end{align}
That concludes the proof.
\end{proof}
\section{Conclusions}

In this paper, we have studied the asymptotic thresholds for weak and strong detection in the Gaussian correlated databases detection problem. Our results are summarized in Table~\ref{tab:exact}. Specifically, in the case where $d$ tends to $\infty$, Theorem~\ref{th:UBbobak} and Theorem~\ref{th:LB} prove that $d^{-1}$ is a sharp threshold. To wit, neither weak nor strong detection is possible if $\rho^2\ll d^{-1}$, while strong detection is possible if $\rho^2\gg d^{-1}$. 

\bibliographystyle{alpha}

\begin{thebibliography}{DMWX18}

\bibitem[AQ97]{anderson1997monotoneity}
G~Anderson and S-L Qiu.
\newblock A monotoneity property of the gamma function.
\newblock {\em Proceedings of the American Mathematical Society},
  125(11):3355--3362, 1997.

\bibitem[AT92]{arratia1992cycle}
Richard Arratia and Simon Tavar{\'e}.
\newblock The cycle structure of random permutations.
\newblock {\em The Annals of Probability}, pages 1567--1591, 1992.

\bibitem[BBM05]{Berg2005}
A.C. Berg, T.L. Berg, and J.~Malik.
\newblock Shape matching and object recognition using low distortion
  correspondences.
\newblock In {\em Proc. Computer Vision and Pattern Recognition}, 2005.

\bibitem[BE21]{Bakirtas2021DatabaseMU}
Serhat Bakirtas and Elza Erkip.
\newblock Database matching under column deletions.
\newblock {\em 2021 IEEE International Symposium on Information Theory (ISIT)},
  pages 2720--2725, 2021.

\bibitem[BE22]{Bakirtas2022DatabaseMU}
Serhat Bakirtas and Elza Erkip.
\newblock Database matching under column repetitions.
\newblock {\em ArXiv}, abs/2202.01730, 2022.

\bibitem[CMK18]{10.1109/ISIT.2018.8437908}
Daniel Cullina, Prateek Mittal, and Negar Kiyavash.
\newblock Fundamental limits of database alignment.
\newblock In {\em 2018 IEEE International Symposium on Information Theory
  (ISIT)}, page 651–655. IEEE Press, 2018.

\bibitem[CSS06]{10.5555/2976456.2976496}
Timothee Cour, Praveen Srinivasan, and Jianbo Shi.
\newblock Balanced graph matching.
\newblock In {\em Proceedings of the 19th International Conference on Neural
  Information Processing Systems}, NIPS'06, page 313–320, Cambridge, MA, USA,
  2006. MIT Press.

\bibitem[DCK19]{pmlr-v89-dai19b}
Osman~E. Dai, Daniel Cullina, and Negar Kiyavash.
\newblock Database alignment with gaussian features.
\newblock In Kamalika Chaudhuri and Masashi Sugiyama, editors, {\em Proceedings
  of the Twenty-Second International Conference on Artificial Intelligence and
  Statistics}, volume~89 of {\em Proceedings of Machine Learning Research},
  pages 3225--3233. PMLR, 16--18 Apr 2019.

\bibitem[DCK20]{9174507}
Osman~Emre Dai, Daniel Cullina, and Negar Kiyavash.
\newblock Achievability of nearly-exact alignment for correlated gaussian
  databases.
\newblock In {\em 2020 IEEE International Symposium on Information Theory
  (ISIT)}, pages 1230--1235, 2020.

\bibitem[DMWX18]{GraphAl4}
Jian Ding, Zongming Ma, Yihong Wu, and Jiaming Xu.
\newblock Efficient random graph matching via degree profiles.
\newblock {\em Probability Theory and Related Fields}, 179:29--115, 2018.

\bibitem[DWXY21]{GraphAl2}
Jian Ding, Yihong Wu, Jiaming Xu, and Dana Yang.
\newblock The planted matching problem: Sharp threshold and infinite-order
  phase transition.
\newblock {\em ArXiv}, abs/2103.09383, 2021.

\bibitem[Gan20]{GraphAl7}
Luca Ganassali.
\newblock Sharp threshold for alignment of graph databases with gaussian
  weights.
\newblock In {\em MSML}, 2020.

\bibitem[KHP12]{kang2012fast}
U.~Kang, M.~Hebert, and S.~Park.
\newblock Fast and scalable approximate spectral graph matching for
  correspondence problems.
\newblock {\em Information Sciences}, 2012.

\bibitem[KN22a]{9834731}
Zeynep K and Bobak Nazer.
\newblock Detecting correlated gaussian databases.
\newblock In {\em 2022 IEEE International Symposium on Information Theory
  (ISIT)}, pages 2064--2069, 2022.

\bibitem[KN22b]{nazer2022detecting}
Zeynep K and Bobak Nazer.
\newblock Detecting correlated gaussian databases.
\newblock {\em arXiv preprint arXiv:2206.12011}, 2022.

\bibitem[MMX21]{GraphAl1}
Mehrdad Moharrami, Cristopher Moore, and Jiaming Xu.
\newblock {The planted matching problem: Phase transitions and exact results}.
\newblock {\em The Annals of Applied Probability}, 31(6):2663 -- 2720, 2021.

\bibitem[MWXY21]{GraphAl5}
Cheng Mao, Yihong Wu, Jiaming Xu, and Sophie~H. Yu.
\newblock Testing network correlation efficiently via counting trees.
\newblock 2021.

\bibitem[NS08]{4531148}
Arvind Narayanan and Vitaly Shmatikov.
\newblock Robust de-anonymization of large sparse datasets.
\newblock In {\em 2008 IEEE Symposium on Security and Privacy (sp 2008)}, pages
  111--125, 2008.

\bibitem[NS09]{5207644}
Arvind Narayanan and Vitaly Shmatikov.
\newblock De-anonymizing social networks.
\newblock In {\em 2009 30th IEEE Symposium on Security and Privacy}, pages
  173--187, 2009.

\bibitem[PG11]{GraphAl3}
Pedram Pedarsani and Matthias Grossglauser.
\newblock On the privacy of anonymized networks.
\newblock In {\em Knowledge Discovery and Data Mining}, 2011.

\bibitem[SGE19]{ShiraniISIT}
Farhad Shirani, Siddharth Garg, and Elza Erkip.
\newblock A concentration of measure approach to database de-anonymization.
\newblock In {\em 2019 IEEE International Symposium on Information Theory
  (ISIT)}, page 2748–2752. IEEE Press, 2019.

\bibitem[SXB08]{pmid:18725631}
Rohit Singh, Jinbo Xu, and Bonnie Berger.
\newblock Global alignment of multiple protein interaction networks with
  application to functional orthology detection.
\newblock {\em Proceedings of the National Academy of Sciences of the United
  States of America}, 105(35):12763--8, Sep 2008.

\bibitem[Tam22]{tamir}
Ran Tamir.
\newblock Joint correlation detection and alignment of {G}aussian databases.
\newblock 2022.

\bibitem[Tsy04]{tsybakov2004introduction}
Alexandre~B Tsybakov.
\newblock Introduction to nonparametric estimation, 2009.
\newblock {\em URL https://doi. org/10.1007/b13794. Revised and extended from
  the}, 9(10), 2004.

\bibitem[WXY20]{wu2020testing}
Yihong Wu, Jiaming Xu, and Sophie~H Yu.
\newblock Testing correlation of unlabeled random graphs.
\newblock {\em arXiv preprint arXiv:2008.10097}, 2020.

\bibitem[WXY22]{GraphAl6}
Yihong Wu, Jiaming Xu, and Sophie~H. Yu.
\newblock Settling the sharp reconstruction thresholds of random graph
  matching.
\newblock {\em IEEE Transactions on Information Theory}, 68(8):5391--5417,
  2022.

\end{thebibliography}

\end{document}